\documentclass[letterpaper]{article} % DO NOT CHANGE THIS
\usepackage{aaai2026}  % DO NOT CHANGE THIS
\usepackage{times}  % DO NOT CHANGE THIS
\usepackage{helvet}  % DO NOT CHANGE THIS
\usepackage{courier}  % DO NOT CHANGE THIS
\usepackage[hyphens]{url}  % DO NOT CHANGE THIS
\usepackage{graphicx} % DO NOT CHANGE THIS
\usepackage{natbib}  % DO NOT CHANGE THIS AND DO NOT ADD ANY OPTIONS TO IT
\usepackage{caption} % DO NOT CHANGE THIS AND DO NOT ADD ANY OPTIONS TO IT
\usepackage{algorithm}
\usepackage{newfloat}
\usepackage{listings}
\DeclareCaptionStyle{ruled}{labelfont=normalfont,labelsep=colon,strut=off} % DO NOT CHANGE THIS
\floatstyle{ruled}
\newfloat{listing}{tb}{lst}{}
\floatname{listing}{Listing}
\usepackage{url}            % simple URL typesetting
\usepackage{booktabs}       % professional-quality tables
\usepackage{comment}
\usepackage{amscd}
\usepackage{amsmath}
\usepackage{amssymb}
\usepackage{amsthm}
\usepackage{mathrsfs}
\usepackage{mathtools}
\usepackage{algpseudocode}

\usepackage{tikz}
\usepackage{bm}
\usepackage{subfigure}
\usepackage{multirow}
\newcommand{\set}[1]{\{ #1 \}}
\newcommand{\ind}[1]{\mathbb{I}\left[ #1 \right]}

\renewcommand{\epsilon}{\varepsilon}
\newcommand{\bx}{\bm{x}}
\newcommand{\bbeta}{\bm{\beta}}
\newcommand{\pa}{\operatorname{pa}}
\newcommand{\hpa}{\hat{\operatorname{pa}}}

\newtheorem{proposition}{Proposition}

\theoremstyle{definition}

\usepackage{cleveref}
\crefname{algorithm}{Algorithm}{Algorithms}
\crefname{table}{Table}{Tables}
\crefname{figure}{Figure}{Figures}
\crefname{section}{Section}{Sections}
\crefname{problem}{Problem}{Problems}
\crefname{theorem}{Theorem}{Theorems}
\crefname{proposition}{Proposition}{Propositions}
\crefname{lemma}{Lemma}{Lemmas}
\crefname{remark}{Remark}{Remarks}

\title{
Sparse Additive Model Pruning for Order-Based Causal Structure Learning
}
\author{
    Kentaro Kanamori, %\textsuperscript{\rm 1}
    Hirofumi Suzuki, %\textsuperscript{\rm 1}
    Takuya Takagi %\textsuperscript{\rm 1}
}
\affiliations{
    Artificial Intelligence Laboratory, Fujitsu Limited\\
    \{k.kanamori, suzuki-hirofumi, takagi.takuya\}@fujitsu.com
}

\begin{document}

\maketitle

\begin{abstract}
Causal structure learning, also known as causal discovery, aims to estimate causal relationships between variables as a form of a causal directed acyclic graph (DAG) from observational data. One of the major frameworks is the order-based approach that first estimates a topological order of the underlying DAG and then prunes spurious edges from the fully-connected DAG induced by the estimated topological order. Previous studies often focus on the former ordering step because it can dramatically reduce the search space of DAGs. In practice, the latter pruning step is equally crucial for ensuring both computational efficiency and estimation accuracy. Most existing methods employ a pruning technique based on generalized additive models and hypothesis testing, commonly known as CAM-pruning. However, this approach can be a computational bottleneck as it requires repeatedly fitting additive models for all variables. Furthermore, it may harm estimation quality due to multiple testing. To address these issues, we introduce a new pruning method based on sparse additive models, which enables direct pruning of redundant edges without relying on hypothesis testing. We propose an efficient algorithm for learning sparse additive models by combining the randomized tree embedding technique with group-wise sparse regression. Experimental results on both synthetic and real datasets demonstrated that our method is significantly faster than existing pruning methods while maintaining comparable or superior accuracy. 
\end{abstract}

% Uncomment the following to link to your code, datasets, an extended version or similar.
% You must keep this block between (not within) the abstract and the main body of the paper.
% \begin{links}
%     \link{Code}{https://aaai.org/example/code}
%     \link{Datasets}{https://aaai.org/example/datasets}
%     \link{Extended version}{https://aaai.org/example/extended-version}
% \end{links}

\begin{figure*}
    \centering  
    \includegraphics[width=0.8\textwidth]{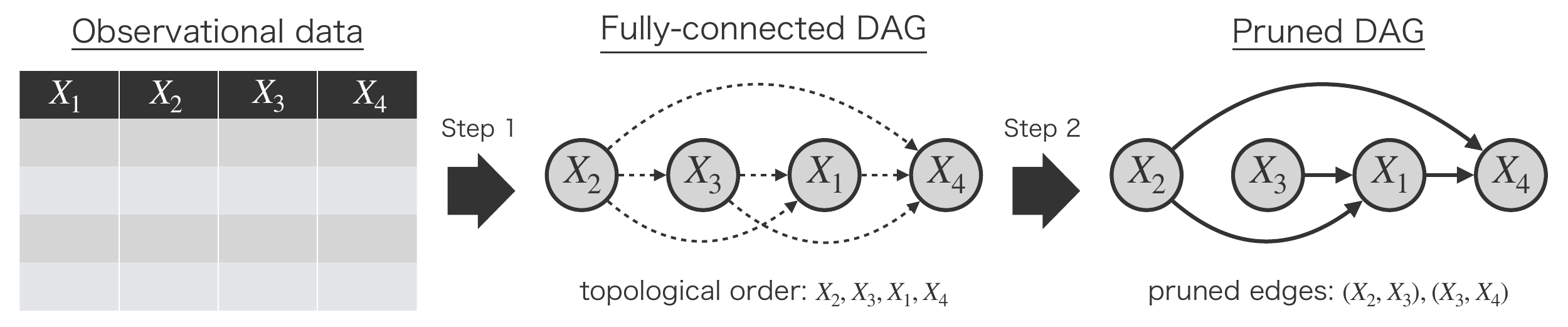}
    \caption{
        An overview of the order-based causal structure learning algorithm. 
        Given an observational dataset, it first estimates a topological order of the underlying causal DAG. 
        Then, it prunes spurious edges from the fully-connected DAG induced by the estimated topological order. 
        This paper focuses on the latter step and aims to propose an efficient and accurate pruning method. 
    }
    \label{fig:intro:overview}
\end{figure*}

\section{Introduction}
In several scientific fields, such as biology and economics, it is often important to identify causal relationships between variables in observational data. 
\emph{Causal structure learning}, also known as \emph{causal discovery}, aims to estimate them as a form of a \emph{causal directed acyclic graph (DAG)}~\cite{Pearl:2009}. 
A causal DAG enables us to understand the relationships between variables and to predict the effect of interventions on the variables, which are crucial for decision-making in various applications~\cite{Peters:2017:causal}.

One of the promising algorithmic frameworks for causal structure learning is the \emph{order-based approach}~\cite{Teyssier:UAI2005}. 
Under some conditions on the data-generating process, previous studies have shown that the underlying causal DAG is identifiable from purely observational data~\cite{Peters:JMLR2014}. 
However, we need to search for a causal graph while ensuring it is acyclic, which incurs super-exponential computational costs in the number of variables~\cite{Chickering:1996}. 
To alleviate this issue, the order-based approach first estimates a topological order of the underlying causal DAG and then prunes spurious edges from the fully-connected DAG induced by the estimated topological order, as illustrated in \cref{fig:intro:overview}. 
By estimating a topological order in advance, we can obtain a causal DAG without explicitly imposing the acyclicity constraint, which reduces the search space dramatically~\cite{Rolland:ICML2022}. 

While existing studies often focus on the former ordering step, in practice, the latter pruning step is equally crucial for ensuring both computational efficiency and estimation accuracy. 
Most existing methods employ a pruning algorithm based on generalized additive models (GAMs)~\cite{Hastie:SS1986}, known as \emph{CAM-pruning}~\cite{Buhlmann:AS2014}. 
It first fits a GAM for regressing each variable on its candidate parents, and then identifies redundant parents by hypothesis testing for the fitted GAM~\cite{Marra:CSDA2011}. 
However, the computational cost of fitting GAMs is generally expensive, especially for high-dimensional cases. 
Since CAM-pruning needs to repeatedly fit GAMs for all variables, it often becomes the bottleneck of the entire algorithm~\cite{Rolland:ICML2022,Montagna:CLEAR2023}. 
Furthermore, CAM-pruning also requires repeating hypothesis testing, which can degrade the estimation accuracy due to multiple testing~\cite{Huang:KDD2018}. 
% These facts suggest that the existing order-based algorithms have room for improvement by developing an efficient pruning method that can directly prune redundant edges without hypothesis testing. 

The goal of this paper is to propose an alternative pruning method that addresses the aforementioned limitations of existing pruning methods. 
% Based on these observations, our goal is to propose an alternative pruning method that addresses the aforementioned limitations of existing pruning methods. 
It enables us to accelerate the existing order-based causal structure learning algorithms without compromising their estimation quality.

\subsubsection{Our Contributions}
In this paper, we propose a new pruning method for order-based causal structure learning. 
Our key idea is to learn a \emph{sparse additive model}~\cite{Ravikumar:JRSSB2009} that regresses each variable on its candidate parents, which enables us to directly prune redundant candidate parents without requiring hypothesis testing. 
To accelerate the process of fitting sparse additive models, we propose a new framework, named \emph{Sparse Additive Randomized TRee Ensemble~(SARTRE)}, by combining the randomized tree embedding and group-wise sparse regression techniques. 
% An overview of our SARTRE framework is illustrated in \cref{fig:alg:overview}.
% 
Our contributions are summarized as follows:
\begin{itemize}
    \item 
    We introduce a new efficient framework for learning a sparse additive model, named SARTRE. 
    We consider a special case of the additive model, where the shape function for each variable is expressed as a linear combination of weighted indicator functions over a set of intervals. 
    We propose to generate a set of intervals by \emph{randomized tree embedding}~\cite{Moosmann:NIPS2007}, and show that we can efficiently learn the sparse weight vector via \emph{group lasso regression}~\cite{Yuan:JRSS2006}. 
    \item 
    We propose an efficient pruning method for order-based causal structure learning by leveraging our SARTRE framework. 
    Given an estimated topological order, our method can efficiently prune redundant edges from the fully-connected DAG induced by the estimated topological order without requiring hypothesis testing. 
    Our method can be combined with any causal ordering algorithm, such as SCORE~\cite{Rolland:ICML2022}. 
    \item 
    By numerical experiments on synthetic and real datasets, we demonstrated that our method achieved a significant speedup compared to existing pruning methods, including CAM-pruning~\cite{Buhlmann:AS2014}. 
    Furthermore, our method achieved comparable or superior accuracy to existing methods, suggesting that it can be a promising alternative to current pruning methods. 
\end{itemize}

\section{Preliminaries}

\subsection{Causal Structure Learning}
\emph{Causal structure learning}, also known as \emph{causal discovery}, aims to estimate a \emph{causal graph} that expresses the causal relationships between variables from observational data. 
For a set of variables $[d] \coloneqq \set{1, \dots, d}$, we consider an \emph{additive noise model~(ANM)} with a directed acyclic graph (DAG) $\mathcal{G}$~\cite{Peters:JMLR2014}. 
More precisely, we assume that a random variable $X = (X_1, \dots, X_d) \in \mathbb{R}^d$ is generated by the following \emph{structural equation model} for each $i \in [d]$: 
\begin{align*}
    X_i = f_i(X_{\pa(i)}) + \varepsilon_i,
\end{align*}
where $X_{\pa(i)}$ is a vector of variables that are parents of $X_i$ in $\mathcal{G}$, $f_i$ is a deterministic link function, and $\varepsilon_i \sim \mathcal{N}(0, \sigma_i^2)$ is an independent Gaussian noise variable. 

% Existing methods often assume that the link functions $f_i$ are nonlinear and twice continuously differentiable in every component~\cite{Rolland:ICML2022}. 
As with the previous studies~\cite{Peters:JMLR2014,Rolland:ICML2022}, we assume that the link functions $f_i$ are nonlinear and twice continuously differentiable in every component.
Under mild conditions, the ANM defined above is known to be \emph{identifiable}; that is, the underlying causal DAG $\mathcal{G}$ can be uniquely recovered from observational data generated according to the joint distribution of $X$~\cite{Peters:JMLR2014}. 
Motivated by this fact, several algorithms for estimating the causal DAG $\mathcal{G}$ from observational data have been proposed so far~\cite{Buhlmann:AS2014,Montagna:CLEAR2023,Xu:NeurIPS2024}.

\begin{figure*}[t]
    \centering
    \subfigure[Model Overview]{
        \includegraphics[width=0.6125\linewidth]{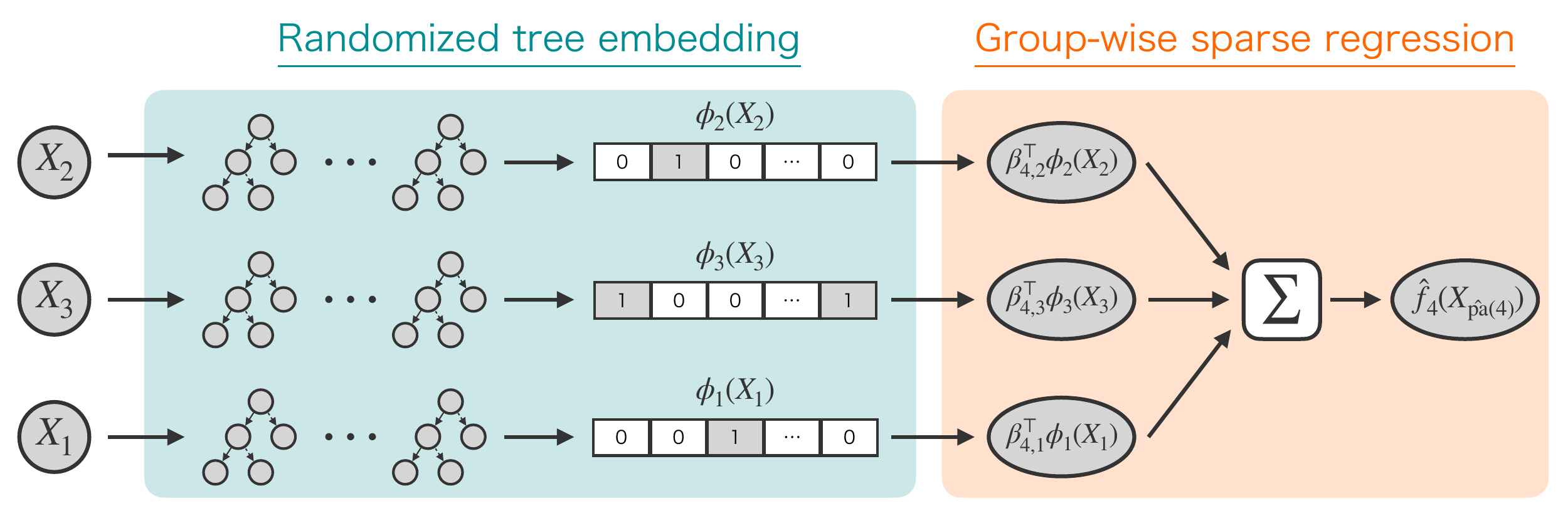}
    }
    \qquad
    \subfigure[Shape Functions]{
        \includegraphics[width=0.2\linewidth]{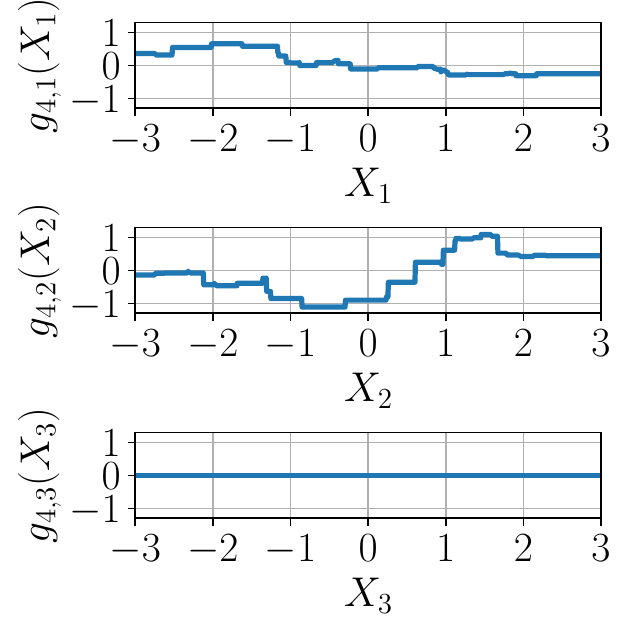}
    }
    \caption{
        An example of our SARTRE framework. 
        We consider the same example as in \cref{fig:intro:overview}, where the estimated topological order $\hat{\pi} = (2, 3, 1, 4)$ is given and we consider to identify the parents of the variable $X_4$ from its candidate parents $\hpa(4) = \set{2, 3, 1}$. 
        (a)~Our method first generates binary representation vectors $(\phi_2(X_2), \phi_3(X_3), \phi_1(X_1))$ by randomized tree embedding.  
        Then, it learns an additive model $\hat{f}_4(X_{\hpa(4)}) = g_{4, 2}(X_2) + g_{4, 3}(X_3) + g_{4, 1}(X_1)$, where each shape function is defined by $g_{4, j}(X_j) = \bbeta_{4, j}^\top \phi_j(X_j)$. 
        By optimizing the coefficient vectors $(\bbeta_{4, 2}, \bbeta_{4, 3}, \bbeta_{4, 1})$ through group lasso regression, we can obtain a sparse additive model $\hat{f}_4$. 
        (b)~For each shape function $g_{4, j}$, if $\bbeta_{4, j} = \bm{0}$ holds, we have $g_{4, j}(X_j) = 0$ for any $X_j$, which enables us to prune the corresponding candidate parent $X_j$. 
        In this example, $\bbeta_{4, 3} = \bm{0}$ holds and thus we can prune $X_3$. 
    }
    \label{fig:alg:overview}
\end{figure*}

\subsection{Order-Based Causal Structure Learning}
One of the major frameworks for estimating a causal DAG $\mathcal{G}$ is the \emph{order-based approach}~\cite{Teyssier:UAI2005}. 
Given an observational sample $S = \set{ \bm{x}_1, \dots, \bm{x}_n} \subseteq \mathbb{R}^d$, it divides the task of estimating $\mathcal{G}$ into two steps: 
(1)~estimating a topological order of $\mathcal{G}$, and
(2)~pruning redundant edges from the fully-connected DAG induced by the estimated topological order. 
\cref{fig:intro:overview} presents an overview of the order-based approach. 
Previous studies often focus on the former step and proposed several methods for estimating a topological order from a given observational sample $S$. 
For the latter pruning step, most of the existing methods employ a nonlinear variable selection algorithm based on \emph{generalized additive models (GAMs)}~\cite{Hastie:SS1986}. 

\subsubsection{Ordering Step}
Given a sample $S$, the ordering step aims to estimate a topological order $\pi$ of the underlying DAG $\mathcal{G}$. 
A topological order $\pi$ is expressed as a permutation over $[d]$ such that, for any $i, j \in [d]$, $\pi(i) < \pi(j)$ holds if $X_j$ is a descendant of $X_i$ in $\mathcal{G}$. 
Existing methods often estimate a topological order in a bottom-up greedy manner that iteratively identifies a leaf variable, i.e., a variable that is not a parent of any other remaining variables~\cite{Peters:JMLR2014}. 

One of the state-of-the-art methods is the \emph{SCORE} algorithm~\cite{Rolland:ICML2022}, which identifies a leaf variable by leveraging the score matching technique. 
For the underlying joint distribution $p(x)$ of $X$, let $s(x) \coloneqq \nabla \log p(x)$ be the logarithmic gradient of $p(x)$, which is known as the \emph{score function}. 
Then, \cite{Rolland:ICML2022} proved that $X_i$ is a leaf variable if and only if the variance of $\frac{\partial s_i(X)}{\partial x_i}$ is zero. 
Based on this fact, SCORE identifies a leaf variable $X_l$ by
\begin{align*}
    l = \arg {\min}_{i \in [d]} \operatorname{Var}_{X} \left[ \frac{\partial s_i(X)}{\partial x_i} \right]. 
\end{align*}
Note that $\operatorname{Var}_{X} \left[ \frac{\partial s_i(X)}{\partial x_i} \right]$ can be estimated by the second-order Stein gradient estimator~\cite{Li:ICLR2018} with a given sample $S$. 
After identifying a leaf variable $X_l$, SCORE removes it from the set of variables and repeats this procedure for the remaining variables. 
When all variables are removed, we can obtain an estimated topological order $\hat{\pi}$. 

\subsubsection{Pruning Step}
Given an estimated topological order $\hat{\pi}$ and $S$, the pruning step aims to remove spurious edges and identify the true parents of each variable $X_i$ in the underlying DAG $\mathcal{G}$. 
Once a topological order $\hat{\pi}$ is determined, we can construct the fully-connected DAG by adding a directed edge from $X_j$ to $X_i$ for each $j, i \in [d]$ such that $\hat{\pi}(j) < \hat{\pi}(i)$ holds.
Let $\hpa(i) \coloneqq \set{j \in [d] \mid \hat{\pi}(j) < \hat{\pi}(i)}$ be the set of \emph{candidate parents} of $X_i$ with respect to $\hat{\pi}$. 
Then, the task of the pruning step is to select the true parents of each $X_i$ from $\hpa(i)$, which can be regarded as a variable selection problem.
Most existing methods employ a variable selection method based on GAMs, known as \emph{CAM-pruning}~\cite{Buhlmann:AS2014}. 
The idea behind CAM-pruning is to assume that the link function $f_i$ can be expressed as an additive model. 
It first fits an additive model that regresses $X_i$ on its candidate parents $X_{\hpa(i)}$:
\begin{align*}
    \hat{f}_i(X_{\hpa(i)}) = {\sum}_{j \in \hpa(i)} g_{i, j}(X_j),
\end{align*}
where $g_{i, j}$ is a nonlinear \emph{shape function} for each $j \in \hpa(i)$. 
Then, CAM-pruning selects a subset of $\hpa(i)$ by hypothesis testing whether $g_{i, j}$ is significantly different from zero for each $j \in \hpa(i)$. 
That is, it removes each candidate parent $X_j$ if the null hypothesis $g_{i, j}(X_j) = 0$ is accepted. 

While CAM-pruning is known as a powerful method, it faces two critical challenges. 
First, while it needs to fit an additive model for each variable $X_i$, repeating this procedure for all variables can be computationally expensive, especially in high-dimensional cases~\cite{Montagna:CLEAR2023}. 
Second, it requires hypothesis testing for all pairs of each variable and its candidate parent, which can harm the correctness of pruning due to multiple testing~\cite{Huang:KDD2018}. 
To address these issues, the aim of this paper is to propose an alternative pruning method that can be applied to high-dimensional datasets while avoiding multiple testing.

\section{Algorithm}
In this section, we propose an efficient and accurate pruning method for order-based causal structure learning. 
We assume that we are given a topological order $\hat{\pi}$ over variables $[d]$ estimated from an observational sample $S = \set{ \bm{x}_1, \dots, \bm{x}_n} \subseteq \mathbb{R}^d$ by some algorithm. 
Note that our method can be combined with any existing ordering algorithm, including CAM~\cite{Buhlmann:AS2014}, SCORE~\cite{Rolland:ICML2022}, and CaPS~\cite{Xu:NeurIPS2024}. 
Given an estimated topological order $\hat{\pi}$, we aim to identify the true parents $\pa(i)$ for each variable $X_i$ by removing spurious variables from its candidate parents $\hpa(i)$ given by $\hat{\pi}$. 

Our main idea is to learn a \emph{sparse additive model}~\cite{Ravikumar:JRSSB2009} that regresses each variable $X_i$ on a small subset of its candidate parents $X_{\hpa(i)}$. 
By learning a sparse additive model, we can directly prune redundant candidate parents without relying on hypothesis testing. 
However, fitting a sparse additive model for each variable $X_i$ can be computationally expensive, as is the case with CAM-pruning. 
Since we need to repeat this procedure for all variables, the computational cost can be prohibitive in practice. 
To address this issue, we propose a new framework for efficiently learning sparse additive models by combining the \emph{randomized tree embedding}~\cite{Moosmann:NIPS2007} and \emph{group lasso regression}~\cite{Yuan:JRSS2006}.

\subsection{Sparse Additive Randomized Tree Ensemble}
First, we introduce \emph{Sparse Additive Randomized TRee Ensemble (SARTRE)}, a new framework for learning sparse additive models. 
We consider a special case of the additive model, where each shape function is expressed as a linear combination of weighted indicator functions over a set of intervals. 
More precisely, we consider an additive model $\hat{f}_i(X_{\hpa(i)}) = \sum_{j \in \hpa(i)} g_{i, j}(X_j)$ whose shape functions $g_{i, j}$ are defined as follows:
\begin{align*}
    g_{i, j}(X_j) = {\sum}_{k=1}^{l_j} \beta_{i, j, k} \cdot \phi_{j, k}(X_j),
\end{align*}
where $l_j$ is the total number of intervals for $X_j$, $\phi_{j, k}(X_j) = \ind{X_j \in r_{j, k}}$ is an indicator function with respect to the $k$-th interval $r_{j, k} \subset \mathbb{R}$, and $\beta_{i, j, k}$ is a coefficient for $r_{j, k}$. 
Each interval $r_{j, k}$ is expressed as $r_{j, k} = (a_{j, k}, b_{j, k}]$ with lower and upper bounds $a_{j, k}, b_{j, k} \in \mathbb{R}$. 
For notational convenience, we denote by $\phi_j(X_j) \coloneqq (\phi_{j, 1}(X_j), \dots, \phi_{j, l_j}(X_j)) \in \set{0,1}^{l_j}$, which can be regarded as an embedding vector of $X_j$ using the set of intervals $R_j \coloneqq \set{r_{j, 1}, \dots, r_{j, l_j}}$. 
Then, our shape function can be expressed by $g_{i, j}(X_j) = \bbeta_{i, j}^\top \phi_{j}(X_j)$, where $\bbeta_{i, j} \coloneqq (\beta_{i, j, 1}, \dots, \beta_{i, j, l_j}) \in \mathbb{R}^{l_j}$ is a coefficient vector for the pair of a variable $X_i$ and its candidate parent $X_j$. 

To learn our additive model $\hat{f}_i$, we need to generate a set of intervals $R_j$ and optimize the coefficient vector $\bbeta_{i, j}$ for each $j \in \hpa(i)$. 
By definition, it is easy to see that $g_{i, j}(X_j) = 0$ if $\beta_{i, j, k} = 0$ holds for all $k \in [l_j]$. 
It implies that we can obtain a sparse additive model by forcing $\bbeta_{i, j} = \bm{0}$ for as many candidate parents $j \in \hpa(i)$ as possible. 
Based on this observation, our SARTRE consists of the following two steps: 
(1)~\emph{randomized tree embedding}, which efficiently generates a set of intervals $R_j$ by an unsupervised manner; 
(2)~\emph{group-wise sparse regression}, which optimizes each coefficient vector $\bbeta_{i, j}$ by group lasso regression. 
\cref{fig:alg:overview} illustrates an overview of our SARTRE framework.

\begin{figure}[t]
    \centering  
    \includegraphics[width=\linewidth]{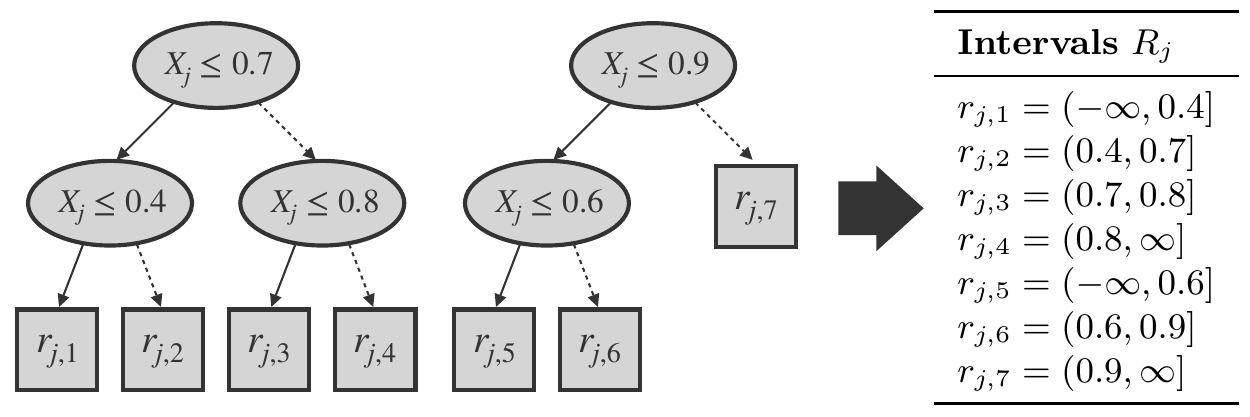}
    \caption{
        An example of the tree embedding technique. 
        Given an ensemble of decision trees that takes $X_j$ as input, each leaf $k$ of a tree corresponds to an interval $r_{j, k}$ of $X_j$. 
        Thus, we can obtain a set of intervals $R_j$ by collecting intervals corresponding to the leaves in a given tree ensemble. 
    }
    \label{fig:alg:tree}
\end{figure}

\subsubsection{Randomized Tree Embedding}
To generate a set of intervals $R_j$ for each $j \in \hpa(i)$, we employ the embedding technique based on tree ensemble models~\cite{Moosmann:NIPS2007}. 
A tree ensemble model consists of decision trees that partition the input space and assign a constant value to each partition.  
Given a tree ensemble $h_j \colon \mathbb{R} \to \mathbb{R}$ that takes one single variable $X_j$ as input, each leaf $k$ of a decision tree in the ensemble corresponds to an interval $r_{j, k}$. 
It suggests that we can obtain $R_j$ by collecting intervals corresponding to the leaves in a given tree ensemble $h_j$. 
Such a technique is known as \emph{tree embedding}~\cite{Moosmann:NIPS2007}, which enables us to easily express nonlinear relationships~\cite{Friedman:AAS2008,Feng:AAAI2018}. 
\cref{fig:alg:tree} illustrates a running example of how to extract a set of intervals $R_j$ from a tree ensemble $h_j$. 

There exist several possible ways to obtain a tree ensemble $h_j$. 
One straightforward approach is to train a tree ensemble $h_j$ that regresses $X_i$ on $X_j$ by some supervised learning algorithm, such as random forests~\cite{Breiman:ML2001} or gradient boosting~\cite{Friedman:AS2000}. 
To accelerate our generating process, we employ an ensemble of completely randomized trees~\cite{Geurts:ML2006}. 
Because these trees are constructed by randomly selecting a split point for each node without any target variable, we can generate $R_j$ more efficiently than the supervised approach. 
Note that the intervals $R_j$ that are generated in such an unsupervised and randomized manner can not always be optimal in terms of the regression performance. 
However, since we optimize the corresponding coefficient vectors $\bbeta_{i, j}$ in the next step, they may not harm the overall performance if we can set the total number of intervals to be sufficiently large~\cite{Geurts:ML2006}. 

\begin{algorithm}[t]
    \caption{SARTRE-pruning}
    \label{alg:sartre}
    \small
    \begin{algorithmic}[1]
        \For{$j = 1, \dots, d$} \Comment{\emph{Randomized Tree Embedding}} \label{line:tree:start}
            \State Fit a randomized tree ensemble $h_j$;
            \State Extract $R_j = \set{r_{j, 1}, \dots, r_{j, l_j}}$ from $h_j$; 
        \EndFor \label{line:tree:end}
        \State Construct a fully-connected DAG $\hat{\mathcal{G}}$ induced by $\hat{\pi}$. 
        \For{$i = 1, \dots, d$} \Comment{\emph{Group-wise Sparse Regression}} \label{line:lasso:start}
            \State $\hat{\bbeta}_i \leftarrow \arg\min_{\bbeta_i \in \mathbb{R}^{p_i}} \mathcal{L}(\bbeta_i) + \lambda \cdot \sum_{j \in \hpa(i)} \|\bbeta_{i, j}\|_2$; 
            \For{$j \in \hpa(i)$}
                \If{$\hat{\beta}_{i, j, 1} = \dots = \hat{\beta}_{i, j, l_j} = 0$}
                    \State Remove the edge $(X_j, X_i)$ from $\hat{\mathcal{G}}$; 
                \EndIf
            \EndFor
        \EndFor \label{line:lasso:end}
        \State \textbf{return} $\hat{\mathcal{G}}$
    \end{algorithmic}
\end{algorithm}

\subsubsection{Group-Wise Sparse Regression}
Given the generated intervals $R_j$ for each $j \in \hpa(i)$, we optimize the corresponding coefficient vectors $\bbeta_{i, j}$ by the group-wise sparse regression technique~\cite{Yuan:JRSS2006,Massias:ICML2018}. 
Let $p_i = \sum_{j \in \hpa(i)} l_j$ be the total number of intervals for all candidate parents of $X_i$. 
We denote the concatenated embedding and coefficient vectors by $\Phi_i(\bx) \coloneqq (\phi_j(x_j))_{j \in \hpa(i)} \in \set{0, 1}^{p_i}$ and $\bbeta_i \coloneqq (\bbeta_{i, j})_{j \in \hpa(i)} \in \mathbb{R}^{p_i}$, respectively. 
Then, our additive model can be expressed as $\hat{f}_i(\bx_{\hpa(i)}) = \bbeta_{i}^\top \Phi_{i}(\bx)$.
It implies that our additive model can be regarded as a linear model over embedding $\Phi_i$ with a coefficient vector $\bbeta_i$ if the intervals $R_j$ are fixed. 
In addition, our coefficient vector $\bbeta_i$ has a group structure, where each group corresponds to one candidate parent of $X_i$, and we aim to encourage group-wise sparsity in $\bbeta_i$ for obtaining a sparse additive model. 

By leveraging the facts mentioned above, we can fit a coefficient vector $\bbeta_i$ by \emph{group lasso regression}~\cite{Yuan:JRSS2006}. 
Our problem can be formulated as follows:
\begin{align*}
    {\min}_{\bbeta_i \in \mathbb{R}^{p_i}} \; \mathcal{L}(\bbeta_i) + \lambda \cdot {\sum}_{j \in \hpa(i)} \|\bbeta_{i, j}\|_2,
\end{align*}
where $\mathcal{L}(\bbeta_i) \coloneqq \sum_{m=1}^{n} \left( x_{m, i} - \bbeta_{i}^\top \Phi_{i}(\bx_{m}) \right)^2$ and $\lambda > 0$ is a regularization parameter. 
While the first term of the objective function is the standard squared loss, the second term is the group lasso penalty that encourages group-wise sparsity in $\bbeta_i$. 
We can solve the above optimization problem by existing efficient algorithms, such as block-coordinate descent~\cite{Friedman:arxiv2010} or dual extrapolation~\cite{Massias:ICML2018}.

\begin{figure*}[t]
    \centering
    \includegraphics[width=0.875\linewidth]{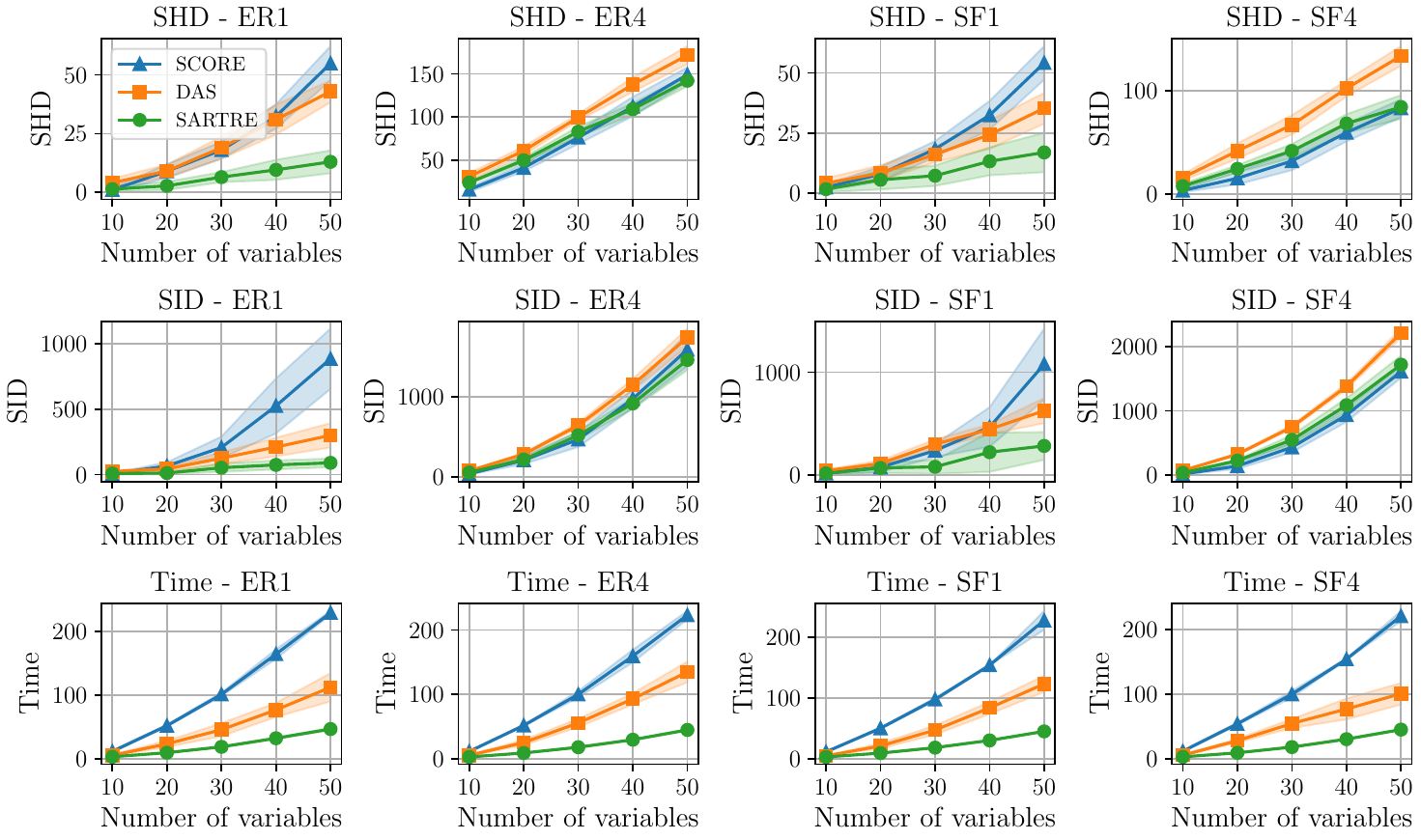}
    \caption{
        Experimental results of baseline comparison on the synthetic datasets. 
        For all the metrics, smaller values are better. 
        The shaded areas indicate the standard deviations over $10$ trials.
        We varied the number of variables $d$ from $10$ to $50$. 
        Our SARTRE was significantly faster than the baselines while maintaining comparable or superior SHD and SID. 
    }
    \label{fig:exp:baseline:variable}
\end{figure*}
\subsection{Overall Pruning Algorithm}
We now propose a pruning algorithm for order-based causal structure learning, based on our SARTRE framework. 
\cref{alg:sartre} presents our algorithm, named \emph{SARTRE-pruning}. 
Given an estimated topological order $\hat{\pi}$ and a sample $S$, it first generates a set of intervals $R_j$ for each variable $X_j$ by randomized tree embedding. 
Then, for each variable $X_i$, it trains a sparse additive model $\hat{f}_i(X_{\hpa(i)})$ by optimizing the coefficient vector $\hat{\bbeta}_i$ through group lasso regression. 
For a candidate parent $j \in \hpa(i)$, if all the coefficients $\hat{\beta}_{i, j, k}$ are zero, we conclude that $X_j$ is not a parent of $X_i$ and prune the edge $(X_j, X_i)$ from the fully-connected DAG $\hat{\mathcal{G}}$ induced by $\hat{\pi}$. 
Finally, our algorithm returns the pruned DAG $\hat{\mathcal{G}}$. 

Our SARTRE-pruning has several advantages compared to the existing pruning methods~\cite{Buhlmann:AS2014,Montagna:CLEAR2023}. 
In terms of computational efficiency, our algorithm generates a set of intervals $R_j$ for each variable in an unsupervised manner (lines~\ref{line:tree:start}--\ref{line:tree:end}) before the pruning procedure.  
Since the generated intervals $R_j$ are independent of any target variable, we can use $R_j$ for all the variables whose candidate parents include $X_j$. 
Thus, all we need in our pruning procedure is to optimize the coefficient vector $\hat{\bbeta}_i$ for each $X_i$ (lines~\ref{line:lasso:start}--\ref{line:lasso:end}), which is more efficient than fitting an additive model from scratch. 
Furthermore, our algorithm obtains a sparse additive model for each variable $X_i$ that encourages as many shape functions $g_{i, j}$ as possible to be zero~\cite{Ravikumar:JRSSB2009}. 
It enables us to directly identify spurious edges without multiple testing.

\subsection{Theoretical Analysis}
Finally, we discuss the representation ability of our additive model compared to the standard one used in CAM-pruning. 
As a shape function, existing implementations of CAM-pruning often employ a \emph{smoothing spline}~\cite{Hastie:2009:Elements}, which is a smooth piece-wise polynomial function.  
On the other hand, our shape function $g_{i, j}$ does not include polynomial terms and only consists of a linear combination of weighted indicator functions over a set of intervals. 
Contrary to such a simple structure, we show that our shape function has the potential to express any continuous function arbitrarily well in \cref{prop:representation}. 

\begin{proposition}\label{prop:representation}
    For a variable $X_j$, we assume $X_j \in [a_j, b_j]$ for some $a_j < b_j$. 
    Then, for any continuous function $g^\ast \colon [a_j, b_j] \to \mathbb{R}$ and $\epsilon > 0$, there exist our shape function $g_{i, j}$ such that $\max_{x_j \in [a_j, b_j]} |g^\ast(x_j) - g_{i, j}(x_j)| < \epsilon$ holds. 
\end{proposition}
\begin{proof}[Proof (sketch)]
    By definition, our shape function $g_{i, j}$ can be expressed as a \emph{piece-wise constant function}, which is a universal approximator for any continuous function~\cite{Cybenko:MCSS1989}. 
    This fact implies the existence of $g_{i, j}$ that can approximate a given continuous function $g^\ast$ arbitrarily well. 
\end{proof}

\cref{prop:representation} suggests that our additive model potentially has a rich representation ability, while it is restricted to a simple shape function compared to CAM-pruning. 
Unfortunately, it is not trivial to determine the sufficient number of intervals $l_j$ and learn the appropriate coefficients $\bbeta_{i, j}$. 
While we leave them for future work, in the next section, we empirically demonstrate that our method works well in practice. 

% On the other hand, the shape function of our additive model can be regarded as a \emph{piece-wise constant function}; that is, our shape function $g_{i, j}$ can be expressed as follows:
% \begin{align*}
%     g_{i, j}(X_j) = {\sum}_{p=1}^{q} \alpha_{p} \cdot \ind{X_j \in (\gamma_{p-1}, \gamma_{p}]},
% \end{align*}
% where $\gamma_{0} < \gamma_{1} < \dots < \gamma_{q}$ are the sorted boundaries of the intervals $R_j$ and $\alpha_{p}$ is a coefficient with respect to the interval $(\gamma_{p-1}, \gamma_{p}]$, which can be computed by combining its corresponding subset of the original coefficients $\bbeta_{i, j}$. 

\begin{figure}[t]
    \centering
    \includegraphics[width=0.83\linewidth]{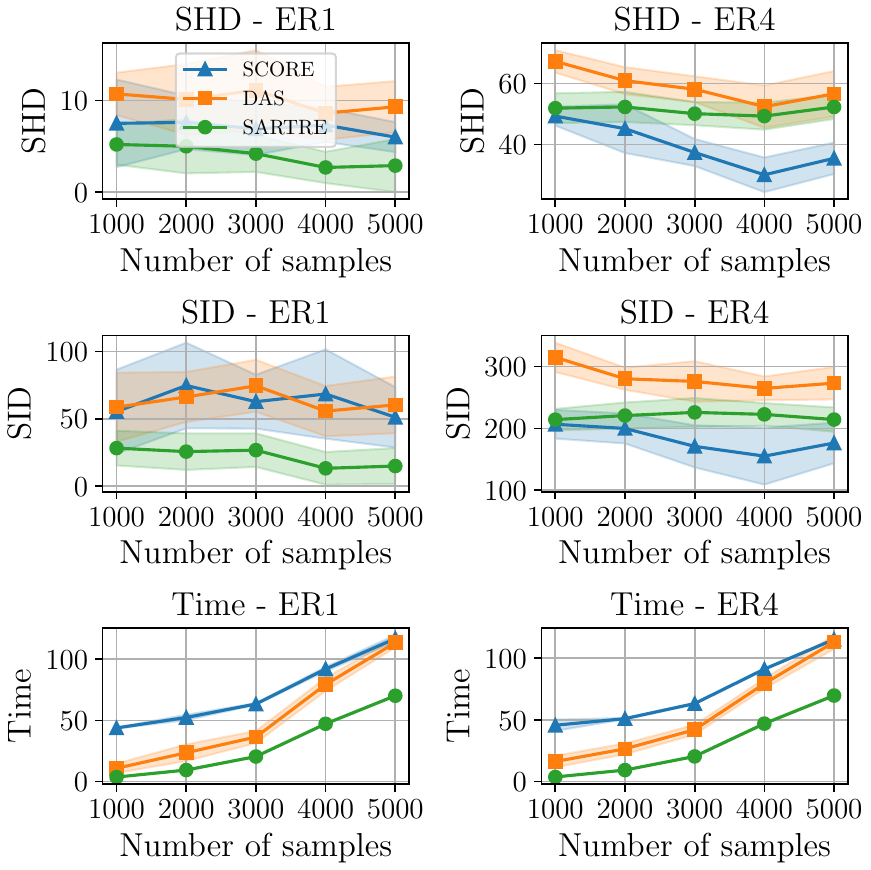}
    \caption{
        Experimental results of baseline comparison by varying the number of samples $n$ from $1000$ to $5000$ on ER1 and ER4 datasets. 
        Our SARTRE was faster than the baselines without significantly degrading SHD and SID. 
    }
    \label{fig:exp:baseline:sample}
\end{figure}
\begin{figure}[t]
    \centering
    \includegraphics[width=0.83\linewidth]{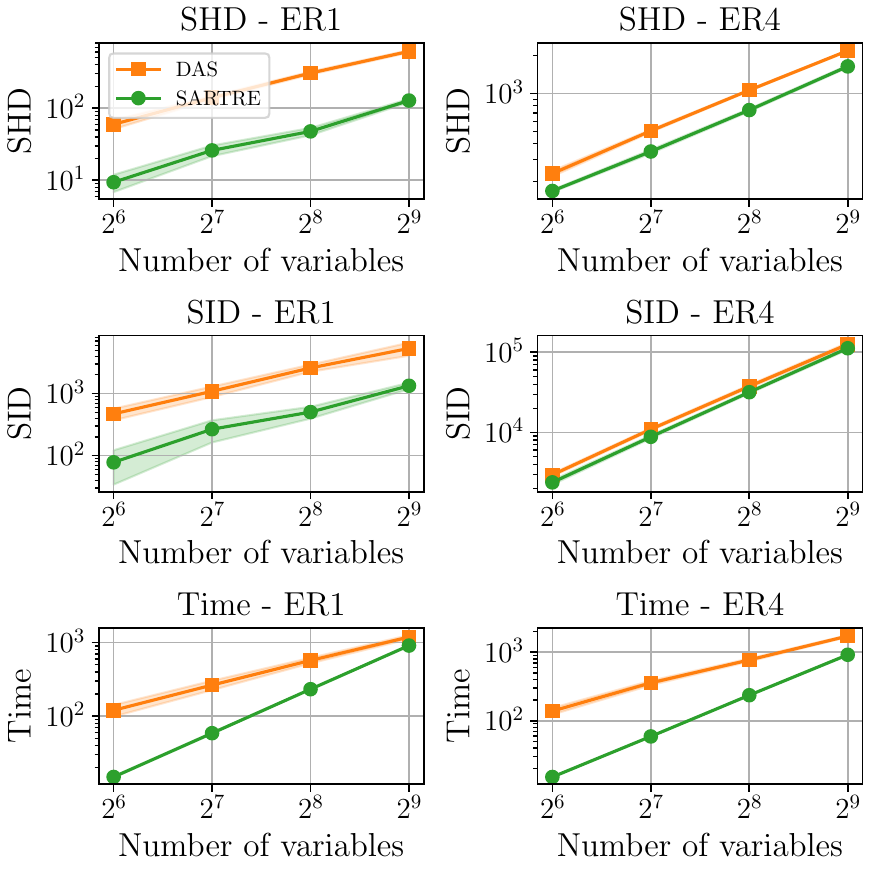}
    \caption{
        Experimental results of the high-dimensional cases on ER1 and ER4 datasets. 
        We varied the number of variables by $d \in \set{64, 128, 256, 512}$. 
        We observed that our SARTRE was faster and attained better SHD and SID than DAS. 
    }
    \label{fig:exp:highdim}
\end{figure}

\section{Experiments}
To investigate the efficacy of our framework, we conducted numerical experiments on synthetic and real datasets. 
All the code was implemented in Python 3.10. 
All the experiments were conducted on macOS Sequoia with Apple M2 Ultra CPU and 128 GB memory. 
Due to page limitations, the complete results are shown in Appendix.

\subsection{Experimental Setup}
\subsubsection{Datasets}
We examine the performance of our method on synthetic datasets generated from a nonlinear ANM with Gaussian noise. 
Following~\cite{Rolland:ICML2022}, we generated link functions $f_i$ by sampling Gaussian processes with a unit bandwidth RBF kernel. 
We considered two types of causal DAGs: \emph{Erdős–Rényi (ER)} and \emph{scale-free (SF)} models. 
For a fixed number of variables $d$, we varied the sparsity of the sampled graph by setting the average number of edges to be $d$ or $4d$. 
We also used real and semi-real datasets: \emph{Sachs}~\cite{Sachs:Science2005} and \emph{fMRI}~\cite{Smith:NI2011}. 

\subsubsection{Baselines}
We compare our method (\emph{SARTRE}) with two existing baselines. 
One baseline is \emph{SCORE}~\cite{Rolland:ICML2022}, which estimates a topological order by the score matching and then prunes spurious edges by CAM-pruning. 
Another baseline is \emph{DAS}~\cite{Montagna:CLEAR2023}, which is a fast variant of SCORE that filters redundant candidate parents before applying CAM-pruning. 
Some other methods, such as CAM~\cite{Buhlmann:AS2014}, NOTEARS~\cite{Zheng:NIPS2018,Zheng:AISTATS2020}, and GraNDAG~\cite{Lachapelle:ICLR2020}, are not included in our comparison since they were outperformed by SCORE in \cite{Rolland:ICML2022}. 
We implemented our SARTRE by combining the ordering step of SCORE and \cref{alg:sartre}. 
Note that all methods use the same ordering step, and their differences are only in the pruning step. 
For SARTRE, we set $\lambda = 0.1$ because it performed the best in our sensitivity analyses, which are shown in Appendix. 
% and present its sensitivity analyses in Appendix. 
We also set the total number of trees and the maximum leaf size of each tree in the ensemble $h_j$ to be $5$ and $8$, respectively, which means that the total number of intervals $l_j$ is at most $40$. 
We repeated each experiment $10$ times and report the average values of (i)~structural Hamming distance~(SHD), (ii)~structural interventional distance~(SID)~\cite{Peters:NC2015}, and (iii)~running time [s] for each method.

\subsection{Experimental Results}
\subsubsection{Baseline Comparison}
First, we evaluate the performance of each method by varying the number of variables $d$. 
\cref{fig:exp:baseline:variable} shows the average SHD, SID, and running time of each method on the synthetic datasets with $n = 2000$. 
We can see that our SARTRE was significantly faster than the baselines, as the number of variables $d$ approached $50$. 
Furthermore, while SARTRE maintained comparable SHD and SID to the baselines on ER4 and SF4, it achieved the best SHD and SID on ER1 and SF1. 
It indicates that SARTRE was more accurate than the baselines when the underlying causal graph was sparse, while it remained competitive on dense graphs. 
In summary, we confirmed that \emph{our SARTRE achieved significant speedup compared to the baselines while maintaining comparable or superior accuracy}. 

We also compare each method by varying the number of samples $n$. 
\cref{fig:exp:baseline:sample} presents the results on ER1 and ER4 datasets with $d = 20$. 
We observed that the running time of DAS approached that of SCORE as the number of samples $n$ increased, while SARTRE remained faster than both baselines. 
In terms of accuracy, SCORE attained better SHD and SID than SARTRE on ER4. 
One possible reason is that hypothesis testing of CAM-pruning remains accurate for large $n$ relatively to $d$ and dense DAGs. 
Nevertheless, SARTRE outperformed SCORE on ER1 and DAS on both ER1 and ER4.
These results indicate that \emph{our SARTRE often performed well with respect to the number of samples as well}.

\subsubsection{High-Dimensional Case}
Next, we examine our method in high-dimensional cases.
We varied $d$ by $\set{64, 128, 256, 512}$ and fixed $n = 2000$. 
Due to the high dimensionality, we omitted the ordering step of each method and used the ground truth topological order as input. 
\cref{fig:exp:highdim} shows the results on ER1 and ER4 datasets for DAS and SARTRE. 
We observed that SARTRE was faster than DAS while maintaining better SHD and SID in all cases. 
From these results, we confirmed that \emph{our SARTRE performed better than the existing scalable algorithm even in the high-dimensional cases}.

\subsubsection{Real Datasets}
Finally, we evaluate our method on real and semi-real datasets. 
We sampled a subset of each dataset by bootstrap sampling with $n = 2000$. 
\cref{tab:exp:real} shows the average SHD, SID, and running time of each method on Sachs and fMRI datasets over 10 trials. 
We observed that SARTRE was faster than the baselines without significantly degrading SHD and SID. 
These results suggest that \emph{our SARTRE performed well on real datasets, as well as synthetic datasets}.

\begin{table}[t]
    \centering
    \small
    \begin{tabular}{lcccccc}
        \toprule
            & \multicolumn{3}{c}{\textbf{Sachs}} & \multicolumn{3}{c}{\textbf{fMRI}} \\
            \cmidrule(lr){2-4} \cmidrule(lr){5-7}
            \textbf{Method} & SHD & SID & Time & SHD & SID & Time \\
        \midrule
            SCORE & 43.2 & 102.4 & 14.7 & 19.6 & 71.8 & 11.4 \\
            DAS & 27.6 & 70.3 & 6.42 & \textbf{11.6} & \textbf{58.6} & 4.75 \\
            SARTRE & \textbf{22.7} & \textbf{58.0} & \textbf{3.62} & 12.9 & 60.0 & \textbf{3.18} \\
        \bottomrule
    \end{tabular}
    \caption{
        Experimental results on real datasets. 
        We sampled each dataset by bootstrap sampling, and repeated this procedure for 10 trials. 
        We can see that our SARTRE was faster than the baselines while maintaining comparable accuracy. 
    }
    \label{tab:exp:real}
\end{table}

\section{Related Work}
\emph{Causal structure learning}, also referred to as \emph{causal discovery}, has been a fundamental task in the field of statistics, machine learning, and artificial intelligence for decades~\cite{Pearl:2009,Peters:2017:causal}. 
Existing studies can be categorized into three main approaches and mixtures of them: \emph{constraint-based}~\cite{Spirtes:SSCR1991}, \emph{score-based}~\cite{Chickering:JMLR2003}, and \emph{functional model-based methods}~\cite{Shimizu:JMLR2006}. 
This paper focuses on functional model-based methods, and in particular, considers the \emph{additive noise model (ANM)}~\cite{Peters:JMLR2014}, which is a popular and widely studied model in the literature. 

To estimate a causal DAG from observational data, several algorithms have been proposed so far~\cite{Shimizu:JMLR2011,Ghoshal:AISTATS2018,Zheng:NIPS2018,Cai:AAAI2018,Lachapelle:ICLR2020,Zheng:AISTATS2020}. 
In this paper, we focus on the \emph{order-based approach}~\cite{Teyssier:UAI2005}, which first estimates a topological order of the underlying causal DAG and then prunes spurious edges from the fully-connected DAG induced by the estimated order. 
Existing studies often focus on the former step and proposed several methods for estimating a topological order, such as \emph{CAM}~\cite{Buhlmann:AS2014}, \emph{RESIT}~\cite{Peters:JMLR2014}, \emph{SCORE}~\cite{Rolland:ICML2022}, and so on~\cite{Xu:NeurIPS2024,Wang:IJCAI2021,Sanchez:ICLR2023}. 
For the latter step, most of these methods employ \emph{CAM-pruning}~\cite{Buhlmann:AS2014} that identifies spurious edges by nonlinear feature selection based on GAMs. 
One exception is \emph{DAS} proposed by \cite{Montagna:CLEAR2023}, which accelerates CAM-pruning by removing redundant candidate parents in advance by leveraging the score matching~\cite{Rolland:ICML2022}. 
In contrast, this paper proposes an alternative pruning method that aims to address the limitations of CAM-pruning in terms of efficiency and accuracy. 
We demonstrated that our method can significantly accelerate existing order-based causal structure learning algorithms without compromising accuracy, enabling us to apply them to more high-dimensional settings.  

Our SARTRE can be regarded as a new \emph{nonlinear variable selection} method, as well as a new learning algorithm for \emph{sparse additive models}~\cite{Ravikumar:JRSSB2009,Haris:JMLR2022}. 
It consists of the \emph{randomized tree embedding}~\cite{Moosmann:NIPS2007,Geurts:ML2006,Feng:AAAI2018} and \emph{group lasso regression}~\cite{Yuan:JRSS2006,Massias:ICML2018}. 
While frameworks for learning GAMs with tree-based shape functions exist~\cite {Lou:KDD2013}, they are not designed to encourage sparsity. 
Since variable selection plays a crucial role in many applications, our SARTRE has the potential to be applied to various tasks beyond causal structure learning~\cite{Marra:CSDA2011}.

\section{Conclusion}
This paper proposed a new pruning method that enhances the efficiency and accuracy of nonlinear order-based causal structure learning. 
To address the limitations of the existing pruning method based on additive models with hypothesis testing, we introduced a new framework, named SARTRE, for learning a sparse additive model by combining the randomized tree embedding and group-wise sparse regression techniques. 
Our method can efficiently learn a sparse additive model for each variable and its candidate parents, which enables us to directly prune redundant edges without hypothesis testing. 
Experimental results demonstrated that our method was significantly faster than the existing methods while maintaining comparable or superior accuracy. 

\subsubsection{Limitations and Future Work}
% There exist several directions for making our method more practical. 
One limitation of our method is that we need to determine some hyperparameters in advance. 
% , such as the regularization parameter $\lambda$ and the number of intervals $l_j$. 
In our experiments, we used the same values for $\lambda$ and $l_j$ across all settings and confirmed that our method stably performed well in all situations. 
However, it would be beneficial to develop a method that can automatically tune these values based on the data. 
Another limitation is the lack of theoretical guarantees for our pruning quality. 
While we showed that our SARTRE has a rich representation ability in \cref{prop:representation}, guaranteeing the correctness of pruning remains an open problem.  
Finally, we need to examine our method in more general and practical settings. 
For example, investigating the performance of our method in the presence of \emph{latent confounders} is important for future work.

\section*{Ethical Statement}
Our proposed method, named SARTRE, is a new pruning algorithm for nonlinear causal structure learning from observational data. 
We believe that our method can be applied to enhance the understanding of complex systems across various fields, including biology, economics, and the social sciences. 
It enables researchers to discover causal relationships in data that may not be readily observable, leading to more informed decision-making. 
We acknowledge that the use of causal structure learning methods can have significant ethical implications, particularly in sensitive areas such as healthcare or criminal justice. 
To prevent potential misuse, we need to ensure that our method is used responsibly and ethically. 
Overall, we believe that our method can have a positive impact on society by enhancing the understanding of complex systems and facilitating better decision-making, provided it is used adequately.

\section*{Acknowledgments}
We wish to thank Yuta Fujishige, Shun Yanashima, and Ryosuke Ozeki for making a number of valuable suggestions. 
We also thank the anonymous reviewers for their insightful comments.

\bibliography{ref}

\newpage
\appendix

\section{Omitted Proof}
\begin{proof}[Proof of Propsition 1]
    First, we show that our shape function $g_{i, j}$ is a piece-wise constant function. 
    For a set of intervals $R_j = \set{r_{j, 1}, \dots, r_{j, l_j}}$ with $r_{j, k} = (a_{j, k}, b_{j, k}]$ for $k \in [l_j]$, let $\Gamma = \set{\gamma_{j, 1}, \dots, \gamma_{j, q_j}}$ be the sorted set of all the boundaries $a_{j, k}, b_{j, k}$ in $R_j$, where $q_j$ is the total number of unique boundaries.     
    Then, we can express $g_{i, j}$ as follows:
    \begin{align*}
        g_{i, j}(X_j) = {\sum}_{p=1}^{q_{j}-1} \alpha_{p} \cdot \ind{X_j \in (\gamma_{p}, \gamma_{p+1}]},
    \end{align*}
    where $\alpha_p$ is the sum of the coefficients $\beta_{i, j, k}$ such that $r_{j, k} \cap (\gamma_{p}, \gamma_{p+1}] \not= \emptyset$. 
    The above result shows that $g_{i, j}$ is regarded as a piece-wise constant function. 

    Next, we show that a piece-wise constant function is a universal approximator of any continuous function over a bounded interval $[a_j, b_j] \subset \mathbb{R}$. 
    For any continuous function $g^\ast: [a_j, b_j] \to \mathbb{R}$ and $\varepsilon > 0$, there exists $\delta > 0$ such that $|g^\ast(x) - g^\ast(x')| < \varepsilon$ for any $x, x' \in [a_j, b_j]$ with $|x - x'| < \delta$. 
    Here, we set $\Gamma = \set{\gamma_{j, 1}, \dots, \gamma_{j, q_j}}$ such that $\gamma_{p+1} - \gamma_{p} < \delta$ for all $p \in [q_j-1]$, $\gamma_1 = a$, and $\gamma_{q_j} = b$ hold. 
    We also set $\alpha_p = g^\ast(\gamma_{p})$ for all $p \in [q_j]$. 
    Then, for any $x_j \in [a_j, b_j]$, there exists $p \in [q_j-1]$ such that $x_j \in (\gamma_{p}, \gamma_{p+1}]$ and $|x_j - \gamma_{p}| \leq \gamma_{p+1} - \gamma_{p} < \delta$. 
    By combining the above two results, we have $|g^\ast(x_j) - g(x_j)| = |g^\ast(x_j) - \alpha_p| = |g^\ast(x_j) - g^\ast(\gamma_{p})| < \varepsilon$. 
    Since it holds for any $x_j \in [a_j, b_j]$, we have $\max_{x_j \in [a_j, b_j]} |g^\ast(x_j) - g(x_j)| < \varepsilon$, which concludes the proof. 
\end{proof}

\section{Complete Experimental Results}
\subsection{Baseline Comparison}
\cref{fig:appendix:exp:baseline:variable,fig:appendix:exp:baseline:sample} show the complete experimental results of the baseline comparison on the synthetic datasets.

\subsection{High-Dimensional Cases}
\cref{fig:appendix:exp:highdim} shows the complete experimental results of the high-dimensional cases. 

\subsection{Mixture of Linear and Nonlinear Link Functions}
To further investigate the performance of our method, we conducted experiments on the cases where the link functions are a mixture of linear and nonlinear functions. 
We generated each link function $f_i$ as either a linear function or a nonlinear function with probability $0.5$. 
Instead of SCORE, we employed \emph{CaPS}~\cite{Xu:NeurIPS2024} as the ordering method since it is a variant of SCORE that can handle both linear and nonlinear link functions. 

\cref{fig:appendix:exp:mixed} shows the experimental results. 
We observed a similar trend to the previous experiments shown in the main paper. 
These results suggest that our SARTRE performed similarly even when the link functions are a mixture of linear and nonlinear functions. 

\subsection{Sensitivity Analysis}
We conducted sensitivity analyses on the regularization parameter $\lambda$ by varying it from $0.1$ to $0.3$ with a step size of $0.05$. 
In addition to the average SHD and SID, we also measured the average values of F1 score, precision, and recall. 

\cref{fig:appendix:exp:sensitivity:feature,fig:appendix:exp:sensitivity:sample} show the experimental results. 
We can see that $\lambda = 0.1$ attained the best SHD, SID, F1 score, and recall in most cases. 
It suggests that there exists a parameter setting for $\lambda$ that can stably achieve better performance across different situations. 
However, we observed that the precision of $\lambda = 0.1$ was often lower than the others, which is likely due to the fact that a larger $\lambda$ leads to a sparser graph. 
We also observed that the performance of $\lambda = 0.1$ was often worse than the others when $n = 1000$, which is likely due to overfitting. 
We leave the further investigation of the sensitivity of $\lambda$ and how to automatically tune $\lambda$ based on a given dataset to future work.

\section{Additional Comments on Existing Assets}
All the code used in our experiments was implemented in Python 3.10 with scikit-learn 1.5.2\footnote{\url{https://github.com/scikit-learn/scikit-learn}}, pyGAM 0.9.1\footnote{\url{https://github.com/dswah/pyGAM}}, and celer 0.7.3\footnote{\url{https://github.com/mathurinm/celer}}. 
Scikit-learn 1.5.2 and celer 0.7.3 are publicly available under the BSD-3-Clause license. 
pyGAM 0.9.1 is publicly available under the Apache License 2.0. 
% All the scripts and datasets are available in the anonymous GitHub repository at \url{https://anonymous.4open.science/r/orbcastle-40F5}. 
All the real datasets used in our experiments are publicly available and do not contain any identifiable information or offensive content. 
As they are accompanied by appropriate citations in the main body, see the corresponding references for more details.
All the experiments were conducted on macOS Sequoia with Apple M2 Ultra CPU and 128 GB memory. 

% \section{Discussion on Societal Impact}
% In this work, we proposed a new method for nonlinear causal structure learning from observational data. 
% We believe that our method can be used to improve the understanding of complex systems in various fields, such as biology, economics, and social sciences. 
% It enables researchers to discover causal relationships in data that may not be easily observable, leading to better decision-making. 
% We acknowledge that the use of causal structure learning methods can have ethical implications, especially when used in sensitive areas such as healthcare or criminal justice. 
% To avoid potential misuse, we need to ensure that our method is used responsibly and ethically. 
% Overall, we believe that our method can have a positive impact on society by improving the understanding of complex systems and enabling better decision-making as long as it is used adequately. 

\begin{figure*}[p]
    \centering
    \includegraphics[width=0.95\linewidth]{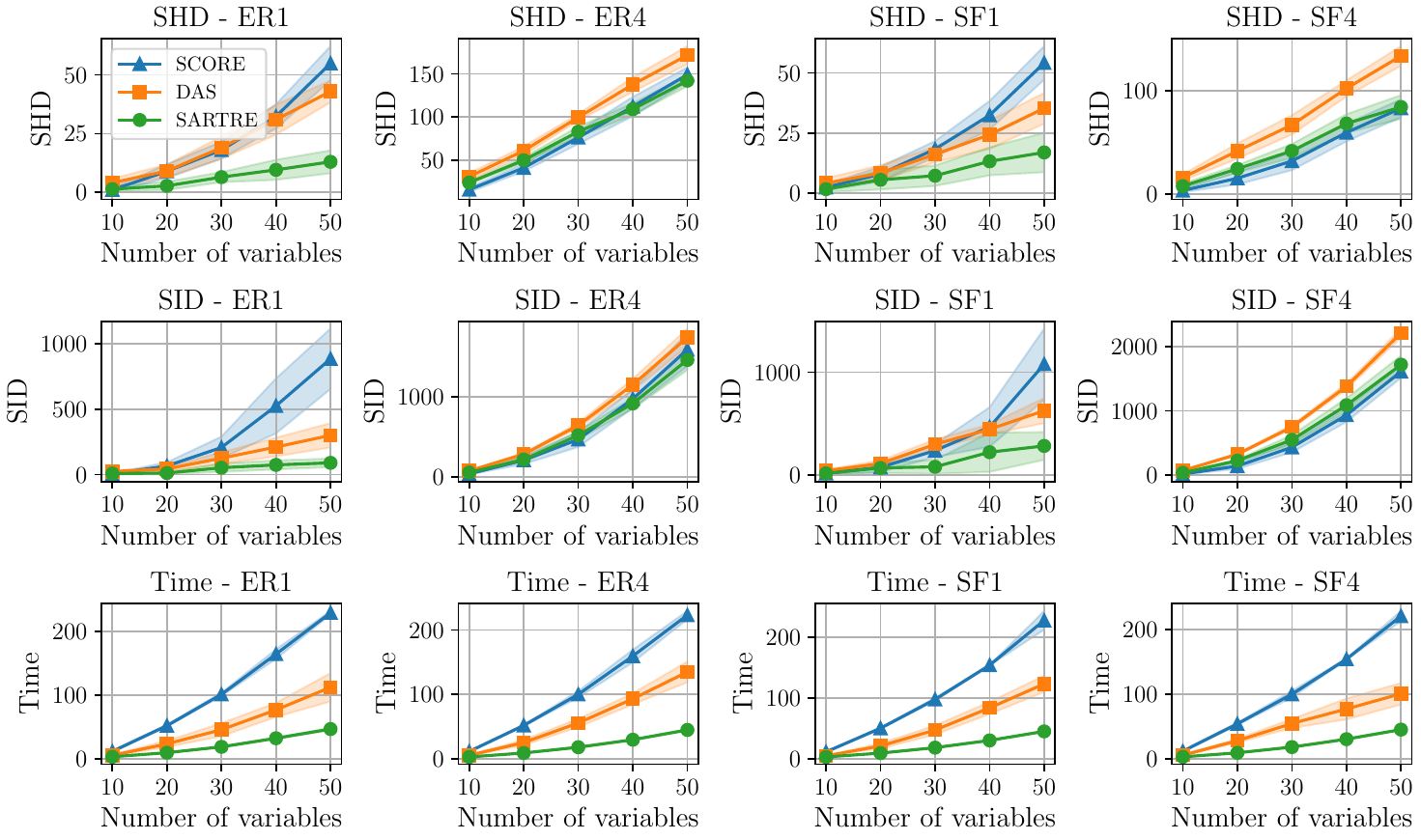}
    \caption{
        Experimental results of baseline comparison on the synthetic datasets. 
        The shaded areas indicate the standard deviations over $10$ trials.
        We varied the number of variables $d$ from $10$ to $50$ and fixed the number of samples $n$ to $2000$.
    }
    \label{fig:appendix:exp:baseline:variable}
\end{figure*}
\begin{figure*}[p]
    \centering
    \includegraphics[width=0.95\linewidth]{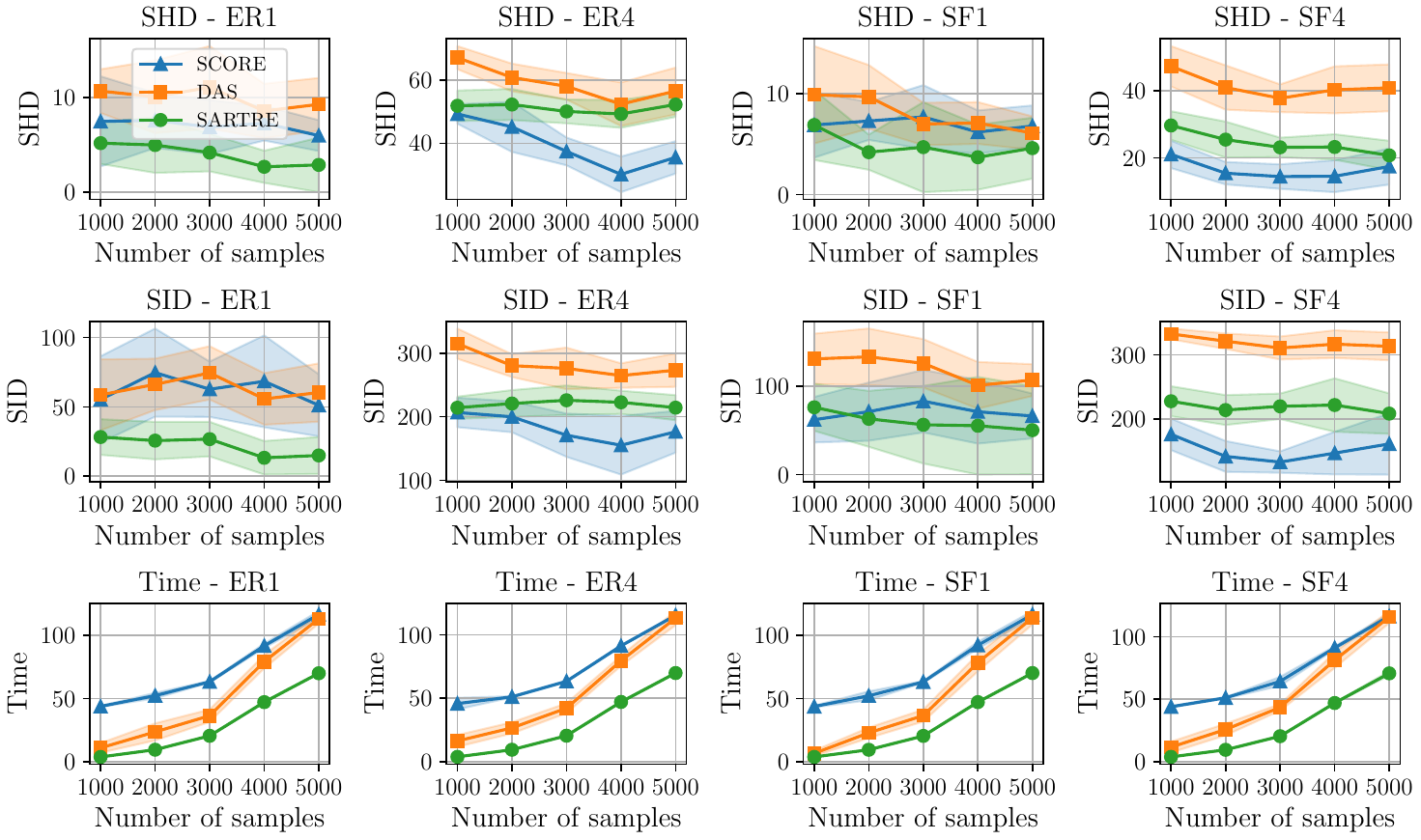}
    \caption{
        Experimental results of baseline comparison on the synthetic datasets. 
        The shaded areas indicate the standard deviations over $10$ trials.
        We varied the number of samples $n$ from $1000$ to $5000$ and fixed the number of variables $d$ to $20$.
    }
    \label{fig:appendix:exp:baseline:sample}
\end{figure*}
\begin{figure*}[p]
    \centering
    \includegraphics[width=\linewidth]{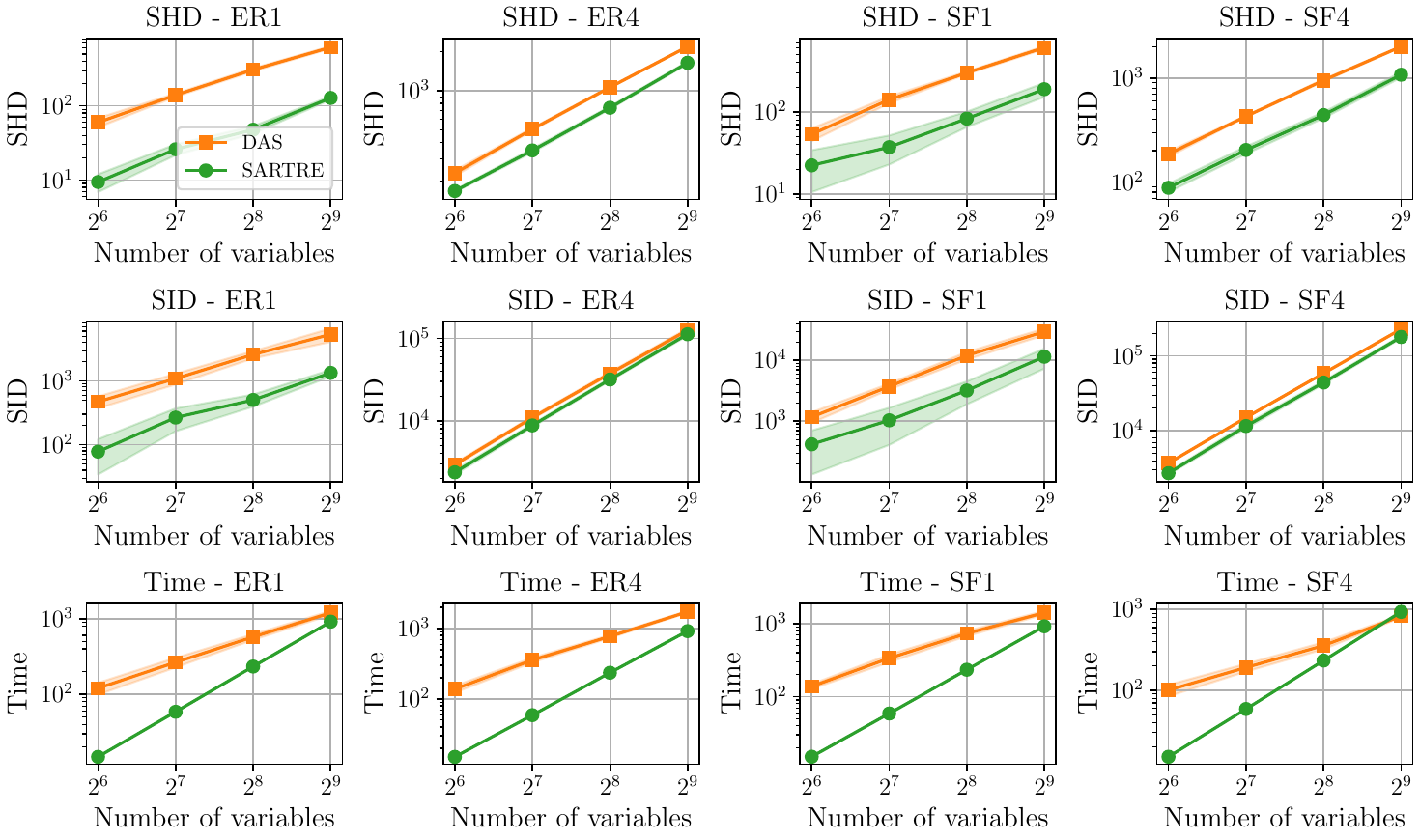}
    \caption{
        Experimental results of the high-dimensional cases on ER1 and ER4 datasets. 
        The shaded areas indicate the standard deviations over $10$ trials.
        We varied the number of variables $d \in \set{64, 128, 256, 512}$ and fixed the number of samples $n$ to $2000$. 
    }
    \label{fig:appendix:exp:highdim}
\end{figure*}
\begin{figure*}[p]
    \centering
    \subfigure[Variables]{
        \includegraphics[width=0.9\linewidth]{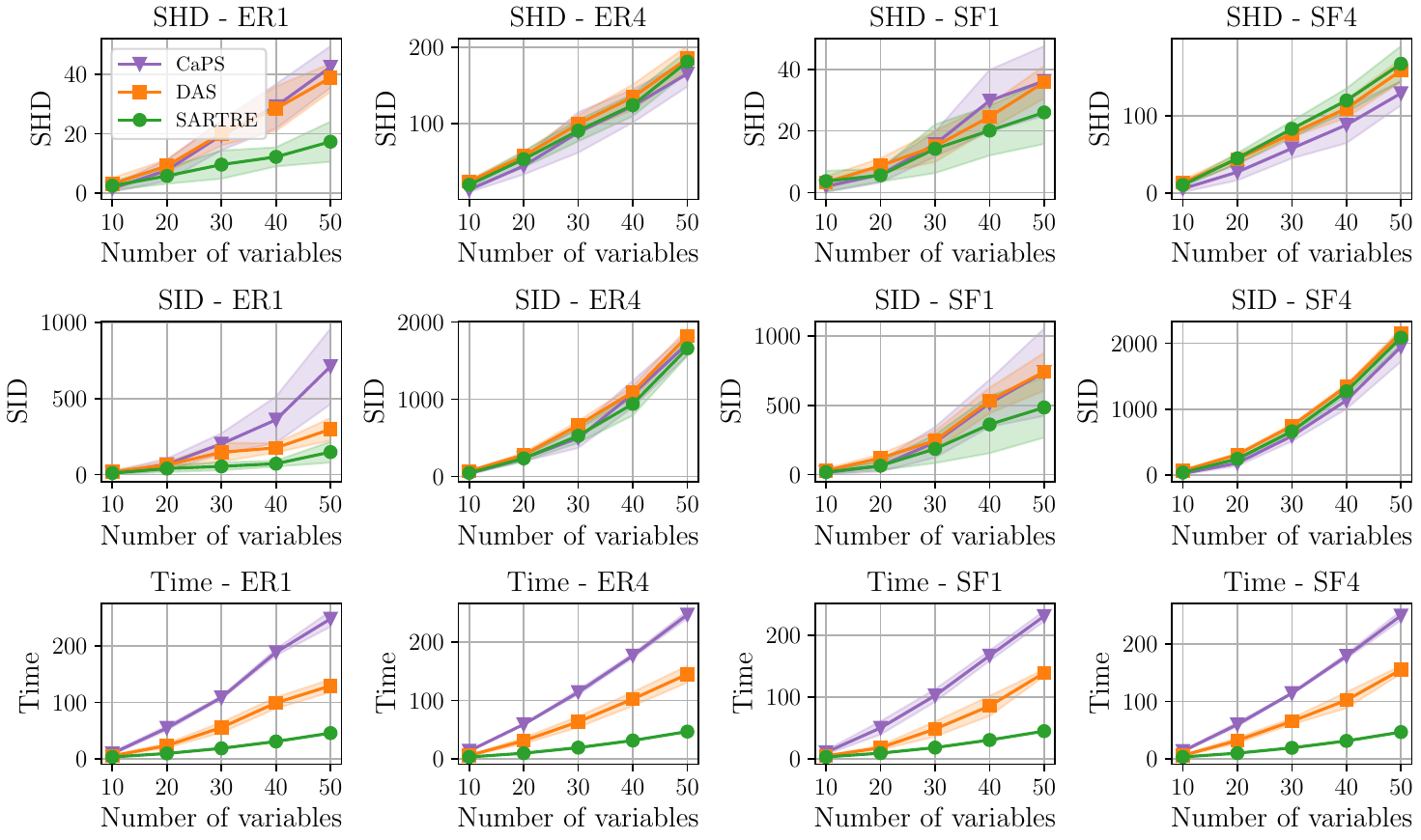}
    }
    \subfigure[Samples]{
        \includegraphics[width=0.9\linewidth]{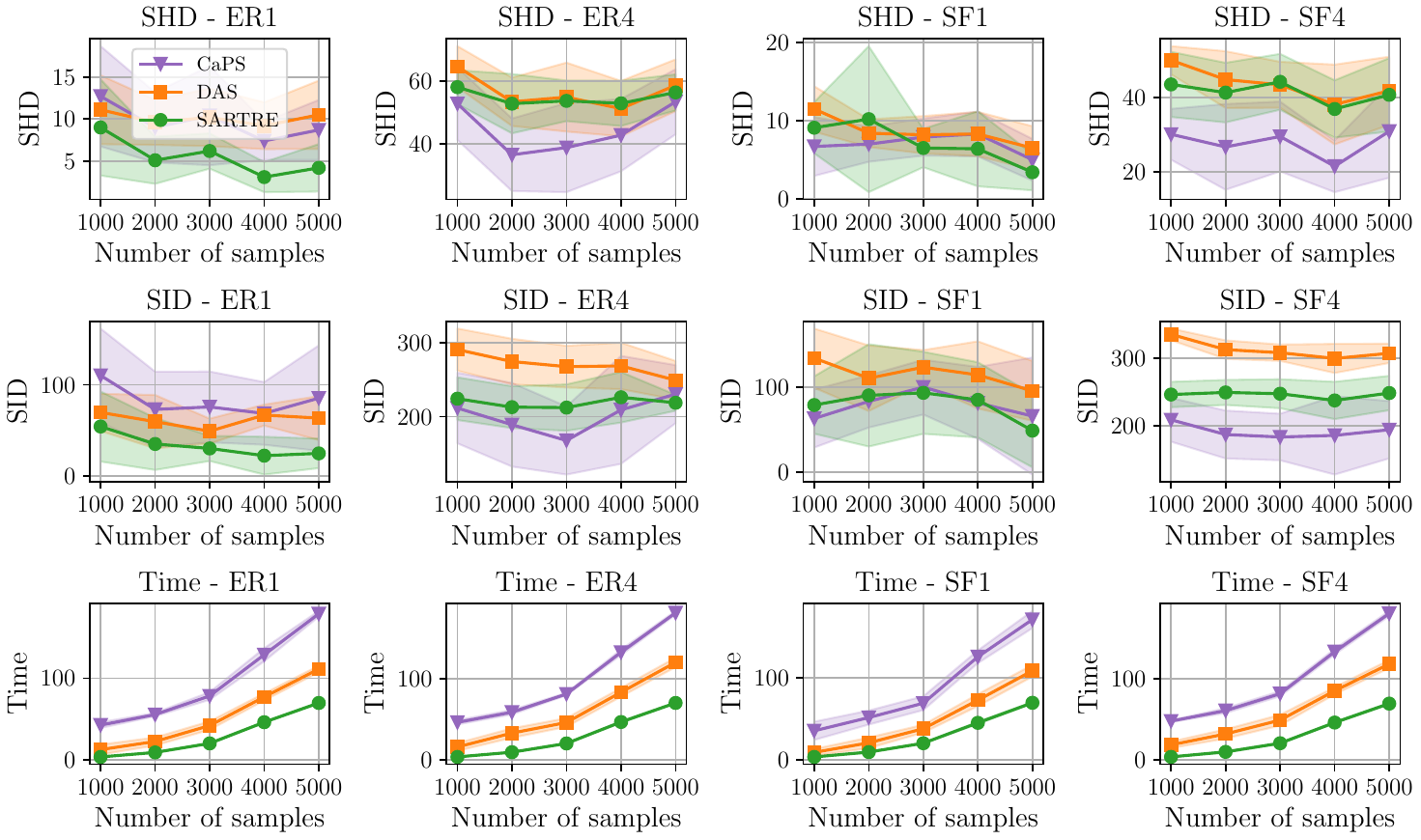}
    }
    \caption{
        Experimental results of the mixed cases. 
        We generated each link function $f_i$ as either a linear function or a nonlinear function with probability $0.5$. 
        The shaded areas indicate the standard deviations over $10$ trials.
        While we varied the number of variables $d$ from $10$ to $50$ and fixed the number of samples $n$ to $2000$ in (a),
        we varied the number of samples $n$ from $1000$ to $5000$ and fixed the number of variables $d$ to $20$ in (b). 
        We observed that our SARTRE performed well even when the link functions are a mixture of linear and nonlinear functions. 
    }
    \label{fig:appendix:exp:mixed}
\end{figure*}
\begin{figure*}[p]
    \centering
    \includegraphics[width=0.95\linewidth]{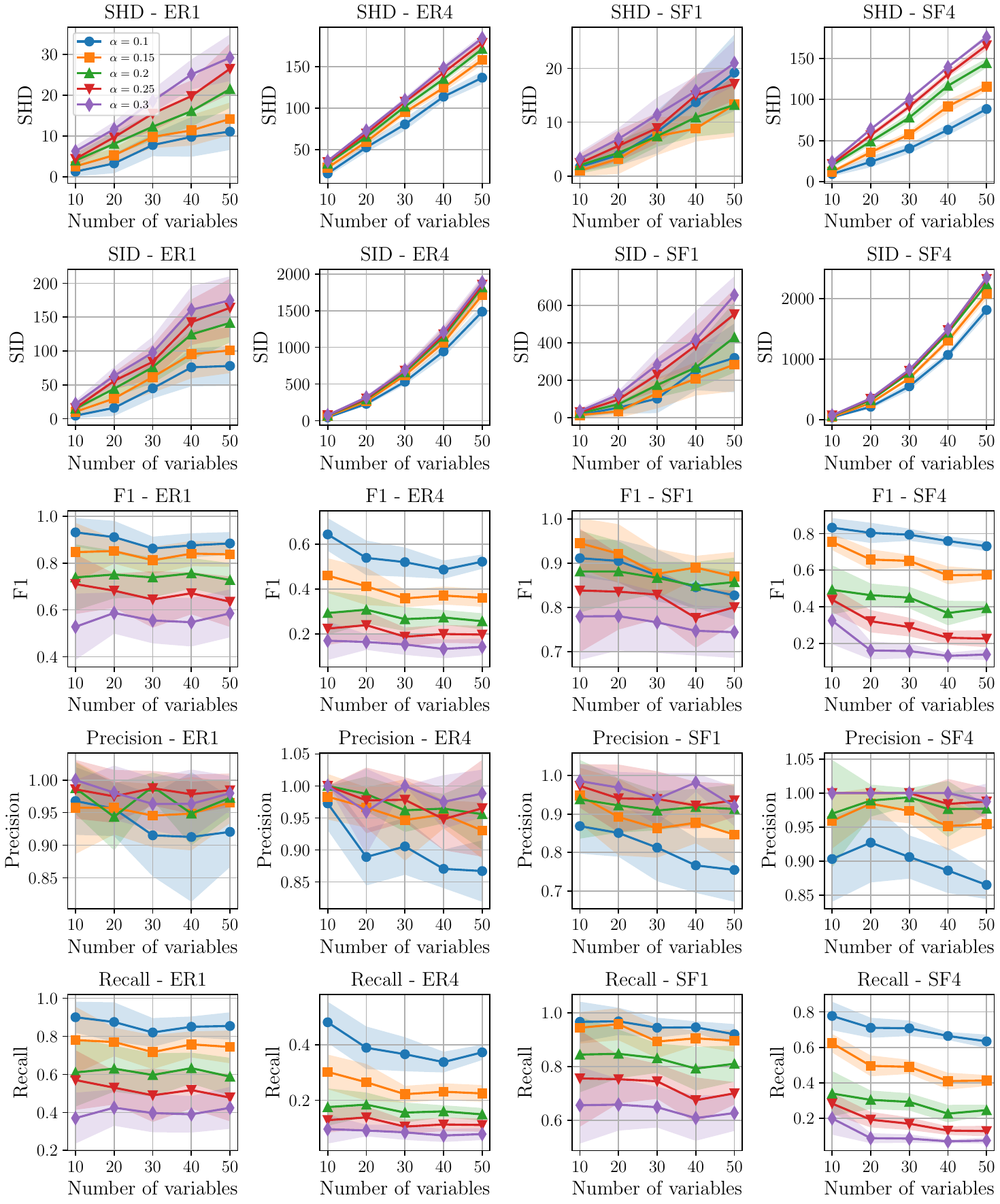}
    \caption{
        Experimental results of the sensitivity analysis on the regularization parameter $\lambda$. 
        The shaded areas indicate the standard deviations over $10$ trials.
        We varied the number of variables $d$ from $10$ to $50$ and fixed the number of samples $n$ to $2000$. 
    }
    \label{fig:appendix:exp:sensitivity:feature}
\end{figure*}
\begin{figure*}[p]
    \centering
    \includegraphics[width=0.95\linewidth]{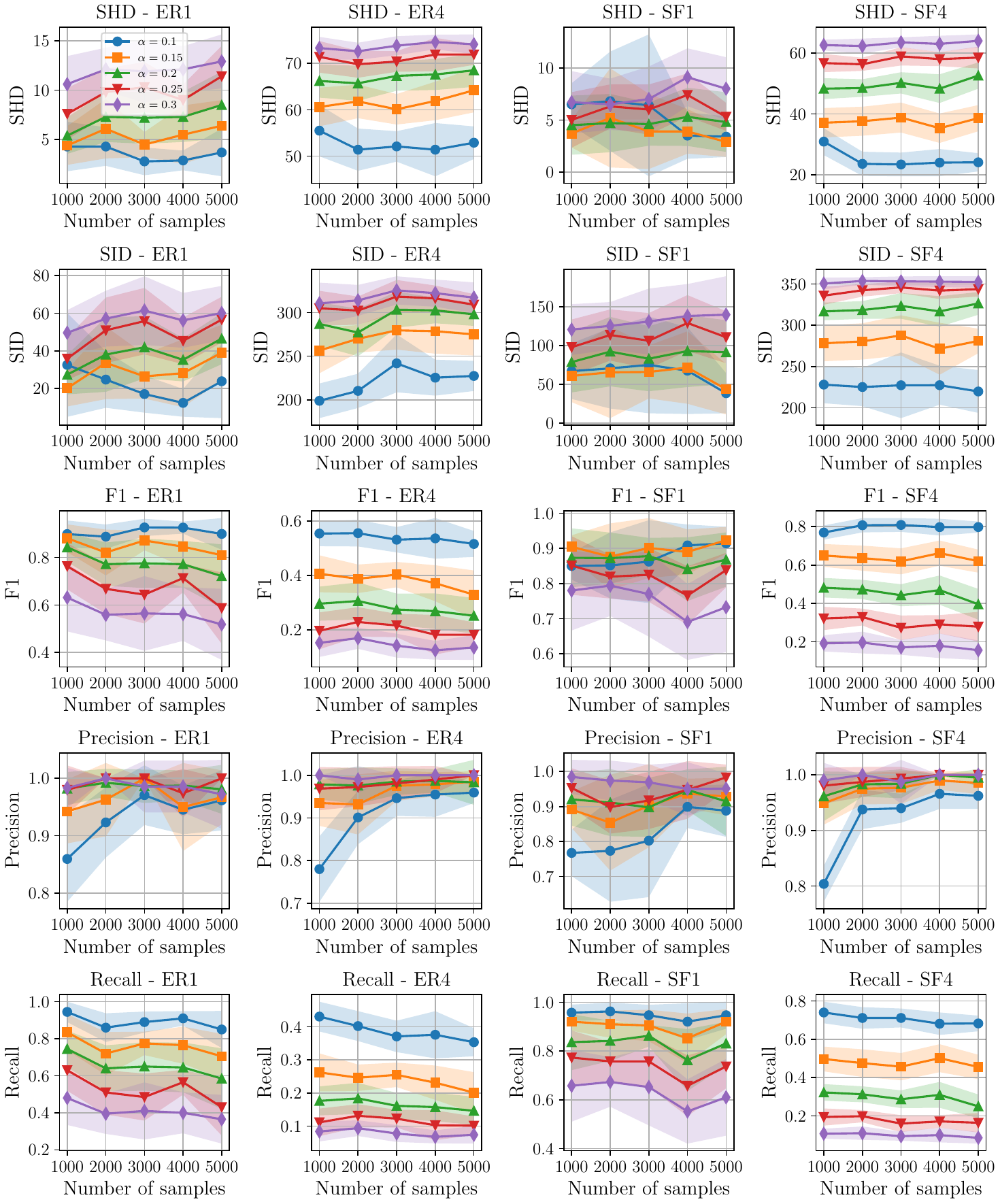}
    \caption{
        Experimental results of the sensitivity analysis on the regularization parameter $\lambda$. 
        The shaded areas indicate the standard deviations over $10$ trials.
        We varied the number of samples $n$ from $1000$ to $5000$ and fixed the number of variables $d$ to $20$. 
    }
    \label{fig:appendix:exp:sensitivity:sample}
\end{figure*}

\end{document}